\documentclass[12pt]{article}
\usepackage[margin=1.05in]{geometry}
\usepackage[utf8]{inputenc}

\usepackage[english]{babel}
\usepackage[round]{natbib}
\usepackage{authblk}
\usepackage[symbol]{footmisc}
\usepackage{algorithm}
\usepackage{algpseudocode}

\usepackage[dvipsnames]{xcolor}

\usepackage{amsmath, amssymb, amsthm,mathtools,dsfont}
\usepackage{bbm}
\usepackage{natbib}
\usepackage{blkarray}
\usepackage{cancel}
\usepackage{enumitem}
\usepackage{euscript}
\usepackage{float}
\usepackage{bm}
\usepackage{graphicx}
\usepackage{mathrsfs}
\usepackage{microtype}
\usepackage{ragged2e}
\usepackage{sectsty}
\usepackage{thmtools}
\usepackage{tikz}
\usepackage[Symbolsmallscale]{upgreek}

\usepackage{microtype}
\usepackage{graphicx}
\usepackage{subcaption}

\usepackage{hyperref}

\newcommand*{\triplenorm}[1]{{\left\vert\kern-0.25ex\left\vert\kern-0.25ex\left\vert #1
    \right\vert\kern-0.25ex\right\vert\kern-0.25ex\right\vert}}

\renewcommand{\phi}{\varphi}

\usepackage{algorithm}
\usepackage{algpseudocode}
\renewcommand{\tilde}{\widetilde}

\renewcommand{\hat}{\widehat}

\usepackage{multirow}

\usepackage{fancyhdr}

\usepackage[capitalize,noabbrev]{cleveref}

\usepackage[textsize=tiny]{todonotes}
\theoremstyle{plain}
\newtheorem{theorem}{Theorem}[section]
\newtheorem{proposition}[theorem]{Proposition}

\theoremstyle{definition}

\theoremstyle{remark}

\hypersetup{citecolor=purple, urlcolor=cyan, colorlinks=true, linkcolor=blue}

\title{Learning local neighborhoods of non-Gaussian graphical models}

\author[1]{Sarah Liaw\footnote{Correspondence to {\tt{sliaw@caltech.edu}}.}}

\begin{document}
\vspace*{-0.5in}

\begin{center} {\LARGE{{Learning local neighborhoods of non-Gaussian \\ 
graphical models: A measure transport approach}}}

{\large{
\vspace*{.3in}
\begin{tabular}{cccc}
Sarah Liaw$^{1}$, Rebecca Morrison$^{2}$, Youssef Marzouk$^{3}$, Ricardo Baptista$^{1}$\\
\hfill \\
\end{tabular}
{
\vspace*{.1in}
\begin{tabular}{c}
				$^1$California Institute of Technology\\
				$^2$University of Colorado Boulder\\
                $^3$Massachusetts Institute of Technology \\
          \\
\end{tabular} 
}

}}

\end{center}

\begin{abstract}
Identifying the Markov properties or conditional independencies of a collection of random variables is a fundamental task in statistics for modeling and inference. Existing approaches often learn the structure of a probabilistic graphical model, which encodes these dependencies, by assuming that the variables follow a distribution with a simple parametric form. Moreover, the computational cost of many algorithms scales poorly for high-dimensional distributions, as they need to estimate all the edges in the graph simultaneously. In this work, we propose a scalable algorithm to infer the conditional independence relationships of each variable by exploiting the local Markov property. The proposed method, named Localized Sparsity Identification for Non-Gaussian Distributions (L-SING), estimates the graph by using flexible classes of transport maps to represent the conditional distribution for each variable. We show that L-SING includes existing approaches, such as neighborhood selection with Lasso, as a special case. We demonstrate the effectiveness of our algorithm in both Gaussian and non-Gaussian settings by comparing it to existing methods. Lastly, we show the scalability of the proposed approach by applying it to high-dimensional non-Gaussian examples, including a biological dataset with more than 150 variables. \looseness-1
\end{abstract}
\footnotetext{Correspondence to {\tt{sliaw@caltech.edu}}.}

\section{Introduction}\label{sec:intro}
Given a collection of random variables $\mathbf{X} = (X_1, \hdots, X_d)$ with probability measure $\mathbf{\nu}_\pi$ and Lebesgue density $\pi$, discovering the conditional independence relationships of $\mathbf{X}$ is an important task in statistics. These dependencies are represented in a graph as edges $E$ between vertices $V$, which correspond to the variables. The resulting graph structure $\mathcal{G} = (V, E)$ is known as a probabilistic graphical model or a Markov random field. 

Many real-world processes, such as gene expression levels, generate continuous and non-Gaussian data, requiring methods that can handle such distributions. Gene expression is regulated by complex networks involving transcription factors which exhibit nonlinear dynamics, leading to non-Gaussian distributions~\citep{Marko_2012}. Other applications include data of financial market returns and climate variables. These datasets are high-dimensional, and existing algorithms that assume normality may fail to correctly characterize the relevant conditional dependencies.

Structure learning algorithms estimate graphs that capture and summarize the conditional dependencies within a dataset, thereby performing model selection. These algorithms are broadly categorized into global and local methods~\citep{koller2009probabilistic}. Global methods reconstruct the entire graph based on global Markov properties, whereas local methods identify the neighborhood set $Nb(k)$ for each node $k$ in the graph $\mathcal{G}$ based on local Markov properties. The neighborhood set defines variables $X_{Nb(k)}$ such that the conditional density of variable $X_k$ satisfies $\pi(x_k| x_{-k}) = \pi(x_k | x_{Nb(k)})$. In other words, $X_k$ is conditionally independent of variables outside the neighborhood $X_{-Nb(k)}$ given $X_{Nb(k)}$. For distributions with a positive probability density functions, global and local Markov properties are equivalent~\citep{Pearl1986GRAPHOIDSGL}.

In this work, we focus on the following graph recovery problem: given i.i.d.\thinspace samples $\{\mathbf{x}^{i}\}_{i=1}^M$ from an (unspecified and possibly non-Gaussian) probability distribution, identify the neighborhood of each node in the graph, i.e., the local Markov properties. %

Existing algorithms for the graph recovery problem that do not assume Gaussianity, such as Sparsity Identification in Non-Gaussian Distributions (SING), rely on the estimation of the underlying data distribution, e.g., by learning a transport map or generative model for the data~\citep{NIPS2017_ea8fcd92, JMLR:v25:21-0022}. To do so, SING estimates the joint density of $\mathbf{X}$, which is often computationally costly for high-dimensional distributions.

To overcome this challenge, we propose Localized Sparsity Identification for Non-Gaussian Distributions (L-SING), which learns transport maps representing the conditional distributions of each node in parallel. Our approach exploits the local Markov property to construct a probabilistic graphical model that encodes the conditional dependencies of the nodes. We summarize the algorithm as follows: (1) learn a transport map for each node from samples; (2) use the estimated maps to define a matrix $\Omega \in \mathbb{R}^{d \times d}$, referred to as a \emph{generalized precision}, which encodes conditional dependencies between non-Gaussian variables; (3) determine the edge set of the graph from the sparsity of $\Omega$.

We show that L-SING generalizes existing methods, such as neighborhood selection with the Lasso and the nonparanormal approach. Finally, we show L-SING's scalability and effectiveness by empirically evaluating it on benchmark problems and a high-dimensional dataset of gene expression levels in ovarian cancer patients.

\section{Related Work}\label{sec:related_work}

Methods to learn probabilistic graphical models are broadly categorized into parametric and non-parametric methods, which identify either global or local Markov properties.

Parametric methods for estimating Markov properties assume the observed data follows a probability distribution whose density and properties are defined by some parameters. %
For Gaussian random variables, conditional dependence is encoded by the sparsity of the precision (inverse covariance) matrix, where zero entries indicate conditional independence between the corresponding variables. Graph structure learning thus reduces to identifying the non-zero entries of the precision matrix in this setting. A widely used approach for this task is the graphical lasso (GLASSO), which solves an $L_1$-penalized maximum likelihood estimation problem for the precision matrix~\citep{JMLR:v9:banerjee08a}.~\citet{10.1093/biostatistics/kxm045} further improved computational efficiency of GLASSO by introducing a coordinate descent algorithm. On the other hand, parametric methods have also been developed to identify local Markov properties by estimating each variable's neighborhood independently. These include greedy selection strategies~\citep{10.1145/2746539.2746631} and penalized maximum likelihood estimators, such as the neighborhood selection method using Lasso regression~\citep{Meinshausen_2006}.

The relationship between the sparsity of the inverse covariance matrix and conditional independence, which is central to many algorithms, does not immediately generalize to non-Gaussian distributions. To address general distributions, semi-parametric methods, such as Gaussian copulas, have been proposed to model non-Gaussian data. For example,~\citet{JMLR:v10:liu09a} assumed observations are generated from marginal nonlinear transformations of a multivariate Gaussian random vector with known Markov properties. %
However, the class of distributions explicitly described by known copulas is relatively limited. While Gaussian copulas introduce non-Gaussianity, it has been observed that they still preserve aspects of the underlying Gaussian structure~\citep{MORRISON2022104983}. Therefore, algorithms that perform well on Gaussian copula-transformed data may struggle with more complex non-Gaussian dependencies.

Previously,~\citet{JMLR:v25:21-0022, NIPS2017_ea8fcd92} proposed SING, which learns the global Markov structure of continuous and non-Gaussian distributions. SING constructs a lower-triangular transport map to estimate the graph structure and iteratively refines the estimated map until the number of edges converges. To do so, SING estimates a multivariate transport map for the entire set of variables simultaneously, storing $\mathcal{O}(d^2)$ entries and thus becoming computationally and memory intensive in high dimensions. Alternatively, L-SING overcomes these challenges by learning local structure: it independently estimates the  neighborhoods for each variable, eliminating the need to maintain a single global transport map. As a result, it only needs to maintain the subset of variable relevant for each node. %
This local method has a lower memory requirement than the global method in high-dimensional settings. Hence, it allows the map for each node to be more expressive and thereby better capture its conditional distribution, e.g., by incorporating larger neural networks with more layers or more parameters, without exceeding practical memory limits. Moreover, since each node’s neighborhood estimation is independent, L-SING can parallelize the learning process for each variable across multiple cores or machines. This reduces both the run time and the overall memory requirements for learning the graph.

Recently,~\citet{Dong_2022} proposed a neighborhood selection method %
that approximates the conditional density of each variable based on a smoothing spline ANOVA decomposition. In contrast, L-SING %
constructs transport maps to represent arbitrary conditional distributions.  
For certain classes of distributions, constructing transport maps offers a more computationally efficient alternative to direct density estimation.

\section{Transport Maps}\label{sec:tm}

A core step of L-SING is estimating a transport map using samples from $\pi$. Transport maps represent a target random variable as a transformation of a reference random variable, such as a standard normal. Given a target probability measure $\mathbf{\nu}_\pi$ and a reference probability measure $\mathbf{\nu}_{\eta}$, both defined on $\mathbb{R}^d$, a transport map $S\colon \mathbb{R}^d \to \mathbb{R}^d$ is a measurable function that couples these two variables such that the pushforward measure of $\mathbf{\nu}_\pi$ through $S$ is $\mathbf{\nu}_\eta$. We denote the pushforward measure as $S_{\sharp} \mathbf{\nu}_\pi$. This condition implies $S(\mathbf{X}) = \mathbf{Z}$ for $\mathbf{X} \sim \mathbf{\nu}_{\pi}$ and $\mathbf{Z} \sim \mathbf{\nu}_{\eta}$. 

If the measures $\mathbf{\nu}_\pi$ and $\mathbf{\nu}_\eta$ have densities $\pi$ and $\eta$, respectively, we denote the pushforward density as $S_{\sharp}\pi = \eta$. When $S$ is invertible, the measure of $S^{-1}(\mathbf{Z})$ corresponds to the pullback density, denoted as $S^{\sharp}\eta = \pi$. For a diffeomorphism $S$, the pushforward and pullback densities can be expressed using the change-of-variables formula:
\begin{align*}
S^{\sharp}\eta(\mathbf{x}) &= \eta \circ S(\mathbf{x}) \cdot \left|\det \nabla S(\mathbf{x})\right| \\
S_{\sharp}\pi(\mathbf{z}) &= \pi \circ S^{-1}(\mathbf{z}) \cdot \left|\det \nabla S^{-1}(\mathbf{z})\right|,
\end{align*}
where $\det \nabla S(\mathbf{x})$ is the determinant of the Jacobian of the map $S$ evaluated at $\mathbf{x}$, and $\circ$ is the composition operator.

In this work, we choose a standard isotropic Gaussian reference, $\nu_\eta = \mathcal{N}(\mathbf{0}, \mathbf{I_d})$ and solve an optimization problem to learn the transport map $S$ given only samples from the target density $\pi$. When $\pi$ and $\eta$ are strictly positive and smooth densities,~\citet{Baptista_2023, Marzouk_2016} showed how to learn a lower-triangular transport map $S$ which has the form:
\begin{equation}\label{eq:lower-tri-tm}
S(\mathbf{x}) = \begin{bmatrix*}[l]
S^1(x_1) \\
S^2(x_1, x_2) \\
S^3(x_1, x_2, x_3) \\
\vdots \\
S^d(x_1, \ldots, x_d) \\
\end{bmatrix*},
\end{equation}
where $S^k$ is a monotone increasing function of $x_k$ for all $(x_1,\dots,x_{k-1})$, i.e., $\partial_{k} S^k \coloneqq \partial_{x_k} S^k > 0$. A feature of lower-triangular maps~\eqref{eq:lower-tri-tm} is their suitability for causal dependencies. Each component $S^k$ represents the conditional distribution of node $X_k$ given the preceding variables in the specified ordering $(X_1,\dots,X_{k-1})$. In particular, for a reference distribution with independent components whose density factorizes as $\eta(\mathbf{x}) = \prod_{k=1}^d \eta_k(x_k)$, component $S^k$ pushes forward the marginal conditional distribution of $X_k$ to the marginal distribution of $Z_k$. That is 
\begin{equation}\label{eq:knb}
    S^k(x_1,\dots,x_{k-1},\cdot)^\sharp \eta_k(x_k) = \pi(x_k|x_1,\dots,x_{k-1}),
\end{equation}
for all values of the conditioning variables $(x_1,\dots,x_{k-1}).$ Moreover, if $X_k$ is conditionally independent of $X_j$ given $X_{(1:k-1) \setminus j}$, then both the density $\pi(x_k|x_1,\dots,x_{k-1})$ and the map $S^k$ do not depend on $x_j$; see~\citet{JMLR:v19:17-747} for more details on the relationship between Markov properties of $\pi$ and the sparsity of triangular maps.

In L-SING, we exploit the relationship between the sparsity of the map component and conditional dependence. In particular, we seek map components $S^k$ that represent the conditional distribution of each variable $X_k$ given all other variables $X_{-k} \in \mathbb{R}^{d-1}$. Given that the last variable of a triangular transport map depends on all variables, this can be seen as learning a component $S^k \colon \mathbb{R}^{d} \rightarrow \mathbb{R}$ to describe the conditional distribution for $X_k$ as
\begin{equation}\label{eq:knb_2}
S^k(x_{-k},\cdot)^\sharp \eta_k(x_k) = \pi(x_k|x_{-k}).
\end{equation}
Our goal is to extract the neighborhood set $Nb(k) \subset \{1,\dots,d\} \setminus k$ for $X_k$ by learning a map $S^k$ that (sparsely) depends on a subset $Nb(k)$ of its inputs. 

\subsection{Parameterization of the Transport Map}\label{sec:param_tm}
To parameterize the transport map component $S^k$ in~\eqref{eq:knb_2}, we use Unconstrained Monotonic Neural Networks (UMNNs) to construct invertible transformations~\citep{NEURIPS2019_2a084e55}. UMNNs define a strictly monotonic function $U$ as
\begin{equation*}
U(x; \psi) = \int_0^x f(t; \psi) \, dt + \beta,
\end{equation*}
where \(f(t; \psi)\colon \mathbb{R} \rightarrow \mathbb{R}^+\) is a strictly positive parametric function implemented via an unconstrained neural network with parameters \(\psi\), and \(\beta \in \mathbb{R}\) is a learnable bias term. The positivity of $f$ is enforced using activation functions like \(\exp\) and softplus applied to the network's output.

Since the derivative \(U'(x;\psi) = f(x; \psi) > 0\) everywhere, \(U(x, ; \psi)\) is strictly monotonic and therefore invertible. This property makes UMNNs suitable for parameterizing transport maps. Additionally, UMNNs allow $f$ to depend on auxiliary input variables, enabling flexible modeling of monotonic functions without restrictive architectural constraints.

In L-SING, an UMNN is used to parameterize each component of the transport map. In particular, we seek \(S^k \in \mathcal{S}_k\) where $\mathcal{S}_k$ is a class of monotonic functions implemented as a UMNN with $d$ input features $\forall k \in [1, d]$ with \(\partial_{k} S^k > 0\). We note that the partial derivative of \(S^k\) %
only requires a single forward pass through the neural network representing the function $f$. This is useful in L-SING to evaluate the conditional density $(S^k)^\sharp \eta_k$, which requires the derivative of the map component.

\section{Learning the Transport Map}\label{sec:learning_map}

To learn the parameters of $S^k$, we formulate the optimization as the solution to a single convex problem. The objective is to minimize
\begin{equation} \label{eq:KLobjective}
    S^k \mapsto \mathbb{E}_{\pi(x_{-k})}[D_{\textrm{KL}}(S^k(x_{-k},\cdot)_{\sharp} \pi(\cdot|x_{-k}) || \eta_k)],
\end{equation} which is equivalent to maximizing the log-likelihood of $\pi(x_k|x_{-k})$ under the model given by the transport map; proof of this statement and the derivation of the sample-based objective is deferred to Appendix~\ref{app:theoretical_dets}. Simply, the approach aims to find the transport map $S^k$ that transforms the conditional distribution $\pi(x_k|x_{-k})$ into a simpler reference distribution $\eta_k$. For the remainder of this work, we take $\eta_k$ to be standard Gaussian.

Given $M$ samples from $\pi$, a regularized maximum likelihood estimation problem for $S^k$ is given by
\begin{equation}\label{eq:tm}
\begin{aligned}
\min_{S^k} \frac{1}{M} \sum_{i=1}^M \left[\frac{1}{2} (S^k)^2(\mathbf{x}^i) - \log \partial_k S^k (\mathbf{x}^i)\right] + \lambda \Phi(S^k)\
\\
\text{s.t. } S^k \in \mathcal{S}_k; \ \ \partial_k S^k >0, \ \ \ \ \ \  (\pi - \text{a.e.})\\
\end{aligned}
\end{equation}
where $\lambda > 0$ is a regularization parameter and %
$\Phi$ is the regularization penalty term  from~\citet{rosasco2012nonparametric} that is used to promote sparse functional dependence:

\begin{equation}\label{eq:reg}
\Phi(S^k) := %
\sum_{j=1}^d \sqrt{\frac{1}{M}\sum_{i=1}^M \left(\frac{\partial S^k(\mathbf{x}^{i})}{\partial x_j}\right)^2}.
\end{equation}

The optimal parameter $\lambda$ is determined in our experiments 
by minimizing the validation loss for the objective in~\eqref{eq:tm}. %

\subsection{Connections to Existing Methods}
To show the generality of L-SING, we demonstrate the connection of the learning problem above to existing methods for Gaussian and nonparanormal distributions.\\

\noindent \textbf{(Gaussian Case)} For a Gaussian vector $\mathbf{X} \sim \mathcal{N} (\mu, \Sigma)$ with mean $\mu \in \mathbb{R}^d$ and non-singular covariance $\Sigma \in \mathbb{R}^{d\times d}$,~\citet{Meinshausen_2006} proposed a neighborhood selection method for estimating the neighborhood sets $\{Nb(k) : 1 \leq k \leq d\}$ through successive linear regressions. 
The neighborhood selection with Lasso estimates the regression coefficients $\theta^k \in \mathbb{R}^{d-1}$ of $X_k$ given co-variates $X_{-k}$ by solving the regularized optimization problem 
\begin{equation}\label{eq:gauss-min}
    \hat{\theta}^k = \arg\min_{\theta} \|\mathbf{X}_k - \mathbf{X}_{-k}\theta\|_2^2 + \lambda \|\theta\|_1,
\end{equation}
where the rows of $\mathbf{X}_k \in \mathbb{R}^M, \mathbf{X}_{-k} \in \mathbb{R}^{M \times (d-1)}$ contain $M$ data samples, and $\|\theta\|_1 $ is the $L_1$-norm of the coefficient vector, which is used to promote sparsity in the solution. The non-zero entries of the resulting vector $\hat\theta^k$ defines the neighborhood set $Nb(k)$.

\begin{proposition}
For a linear transport map component $S^k$, the  optimization problem in~\eqref{eq:tm} reduces to the optimization problem in~\eqref{eq:gauss-min}  for neighborhood selection with the Lasso. %
\end{proposition}
\begin{proof}
For a linear map %
component %
    $S^k(\mathbf{x}_{1:d}) = \sum_{l=1}^d a_lx_l$ %
where $(a_l) \in \mathbb{R}$ are scalar parameters, the objective in~\eqref{eq:tm} is
\begin{align*}
   & \min_{S^k} \frac{1}{M} \sum_{i=1}^M \left[ \frac{1}{2}(S^k)^2 (\mathbf{x}^i) - \log \partial_k S^k(\mathbf{x}^i) \right] + \lambda \Phi(S^k) \\
    =& \min_{a_{1:d}} \frac{1}{M} \sum_{i=1}^M \left[ \frac{1}{2} \left( \sum_{l=1}^d a_lx_l^i \right)^2  - \log a_k \right] + \lambda\sum_{j=1}^d|a_j|
\end{align*}
Consider the change of variables $b_l = \frac{a_l}{a_k}$ for all $l \neq k$. Then, we can write the optimization problem as

\begin{align*}
  \min_{\mathbf{b}_{-k},a_k} \frac{1}{M} \sum_{i=1}^M \left[ \frac{1}{2} (a_k)^2 \left( \sum_{l \neq k} b_l x_l^i + x_k^i \right)^2 - \log a_k \right] 
  + \lambda\left(\sum_{j \neq k}|b_j a_k|+ |a_k|\right),
\end{align*}
where $\mathbf{b}_{-k} = (b_1,\dots,b_{k-1},b_k,\dots,b_{d-1}).$

We identify the re-regression coefficients $\mathbf{b}_{-k}$ by solving the problem
\begin{align*}
    \min_{\mathbf{b}_{-k}} \left\| \mathbf{b}_{-k} \mathbf{X}_{-k} + \mathbf{X}_k \right\|^2 + \tilde{\lambda} \|\mathbf{b}_{-k}\|_1,
\end{align*}
where $\tilde{\lambda} = \lambda (2M) / |a_k|$. This objective has the form as the problem in~\eqref{eq:gauss-min} up to a sign in the coefficients $\mathbf{b}_{-k}$.
\end{proof}
While the relationship between L-SING and neighborhood selection with Lasso is only shown for Gaussian distributions, L-SING extends beyond Gaussians by using nonlinear functions (e.g., higher-order polynomials) to parameterize the transport map.\\

\noindent \textbf{(Nonparanormal Case)} A random vector \( \mathbf{X} = (X_1, \ldots, X_d) \) has a nonparanormal distribution if there exists a set of functions \( \{f_j\}_{j=1}^d \) such that \( \mathbf{Z} =\mathbf{f}(\mathbf{X}) \sim \mathcal{N}(\mu, \Sigma) \), where \( \mathbf{f}(\mathbf{X}) = (f_1(X_1), \ldots, f_d(X_d)) \).~\citet{JMLR:v10:liu09a} proposed applying GLASSO to the transformed data to estimate the undirected graph from the sparsity pattern of the estimated precision matrix:
\[
\hat{\Theta} = \arg\min_{\Theta \geq 0} \left( \operatorname{Tr}(S(\Tilde{\mathbf{f}})\Theta) - \log \det(\Theta) + \lambda \sum_{j \neq k} |\Theta_{jk}| \right),
\]
where $S(\Tilde{\mathbf{f}})$ is a sample covariance estimator of $\Tilde{\mathbf{f}}(\mathbf{X})$ based on an estimator $\Tilde{\mathbf{f}}$ of $\mathbf{f}$.

\begin{proposition}
Let $\mathbf{X}$ be a random vector following a nonparanormal distribution, i.e.,   
there exists monotonic functions $\mathbf{f}$ 
such that $\mathbf{f}(\mathbf{X}) \sim \mathcal{N}(\mathbf{0}, \Sigma)$ for some strictly positive definite $\Sigma$. Then, the transport map $S(\mathbf{x} ) = \Sigma^{-1/2} \mathbf{f}(\mathbf{x} )$ pushes forward $\mathbf{X}$ to a standard normal random variable.
\end{proposition}
\begin{proof}

To show that this is a valid transport map, we show that 
$S(\mathbf{X}) \sim \mathcal{N} (\mathbf{0}, \mathbf{I})$. First, we recall that the covariance matrix $\Sigma$ is symmetric and positive definite, and hence invertible. For a strictly positive definite $\Sigma$, we can use the Cholesky decomposition (or another matrix square root) to obtain $\Sigma^{-1/2}$. Next, we verify that $S(\mathbf{X} ) \sim \mathcal{N}(\mathbf{0}, \mathbf{I})$. Given that $f(\mathbf{X})$ is Gaussian, we just need to verify the first two moments of the linear transformation of $ \mathbf{f}(\mathbf{X} )$:
\begin{align*}
\mathbb{E}[S(\mathbf{X} )] &= \mathbb{E}[\Sigma^{-1/2} \mathbf{f}(\mathbf{X} )] = \Sigma^{-1/2} \mathbb{E}[\mathbf{f}(\mathbf{X} )] = \mathbf{0}\\
\text{Cov}[S(\mathbf{X} )] %
&= \Sigma^{-1/2} \mathbb{E}[\mathbf{f}(\mathbf{X} )\mathbf{f}(\mathbf{X} )^T] (\Sigma^{-1/2})^T %
= \mathbf{I}.
\end{align*}
\end{proof}
The result above shows that nonparanormal distributions can be characterized using transport maps, and their components define its conditionals. Thus, L-SING includes nonparanormal methods as a special case.

\section{Computing the Generalized Precision}\label{sec:gp}
Finally, we show how to compute a matrix encoding the conditional independence structure using the estimated transport map. Recall that in the Gaussian setting, the inverse covariance matrix $\Sigma^{-1}$ is also the precision matrix,where $\Sigma_{kj}^{-1} = 0 \leftrightarrow X_k\perp\!\!\!\perp X_j \,|\, X_{-kj}$~\citep{NIPS2012_6ba1085b}. For non-Gaussian distributions, we extend this concept using the transport maps that represent the conditional distributions of $\mathbf{X}$. Following~\citet{NIPS2017_ea8fcd92}, we consider the generalized precision $\Omega \in \mathbb{R}^{d \times d}$ with entries $$\Omega_{jk} = \mathbb{E}|\partial_j \partial_k \log \pi(\mathbf{x})| = \mathbb{E}|\partial_j \partial_k \log \pi(x_k|x_{-k})|.$$\citet{JMLR:v19:17-747} showed that $\partial_{j}\partial_{k} \log \pi(\mathbf{x}) = 0$ for all $\mathbf{x}$ $\iff X_j \perp\!\!\!\perp X_k \,|\, X_{-kj}$, so the sparsity of the generalized precision matrix encodes pairwise conditional independence properties (similarly to the precision matrix in the Gaussian setting) for distributions whose log-density is twice differentiable.

Given that the map component $S^k$ characterizes the conditional distribution of the target density, we can express the entries of $\Omega$ in terms of the map as:
\begin{align*}
\Omega_{jk} &= \mathbb{E}|\partial_j\partial_k[\log \eta_k (S^k(\mathbf{x}) ) + \log\partial_k S^k(\mathbf{x})]| \\
&=\mathbb{E}\left|\partial_j\partial_k\left[-\frac{1}{2} (S^k)^2 (\mathbf{x}) + \log \partial_k S^k (\mathbf{x})\right]\right|.
\end{align*}
An estimator of this matrix entry based on $N$ i.i.d.\thinspace samples from $\pi$ is given by
\begin{align*}
\hat{\Omega}_{jk} := \frac{1}{N} \sum_{i=1}^N \left|\partial_j\partial_k \left[-\frac{1}{2} (S^k)^2 (\mathbf{x}^i) + \log \partial_k S^k (\mathbf{x}^i)\right]\right|
\end{align*}
For each variable $k$, we use the computed transport map $S^k$ to determine the neighborhood set $\text{Nb}(k)$ based on the non-zero entries $\hat\Omega_{jk}$ for $j \in \{1,\dots,d\} \setminus k$. %
These entries identify edges in the graphical model. In particular, we say that there exists an edge between variables $X_k$ and $X_j$ if $\Omega_{jk}$ meets the conditional independence criteria:
\[
\begin{aligned}
& X_k \perp\!\!\!\perp X_j \,|\, X_{V\backslash\{kj\}}  \Leftrightarrow X_k \perp\!\!\!\perp X_{j \in V\backslash \text{Nb}(k)} \,|\, X_{\text{Nb}(k)}  \Leftrightarrow \partial_{k}\partial_j \log \pi(\mathbf{x}) = 0 \; \forall \; \mathbf{x} \in \mathbb{R}^d
\end{aligned}
\]
Thus, the estimation of the generalized precision matrix allows us to compute an edge set that reflects the conditional independence structure of the data. We note that for each pair of edges $(j,k)$, we obtain two estimators for the conditional independence based on the dependence of map component $S^j$ on $x_k$ and $S^k$ on $x_j$. To reconcile these two estimates, which are theoretically equal when the estimator for the conditional distributions is correct, we compute a symmetrized version of the generalized precision matrix.

\section{L-SING Algorithm} \label{sec:alg}

In this section, we present L-SING for learning the Markov structure of a continuous and (possibly) non-Gaussian distributions. The complete procedure is outlined in Algorithm~\ref{alg:lsing}. 

\begin{algorithm}
\caption{L-SING Algorithm}\label{alg:lsing}

\begin{algorithmic}[1]

    \State \textbf{Input} i.i.d.\thinspace samples $\{\textbf{x}^i\}_{i=1}^{M+N}$,  transport class $\mathcal{S}_k$.
    \State  \textbf{Output} Generalized precision matrix $\hat{\Omega}$.
    \For{each map component $S^k \in \mathcal{S}_k$}
    \For{fixed number of epochs}
    \State Compute the regularized loss using $M$ samples
    \State Back-propagate loss and update $S^k$ parameters
    \EndFor
    \State Compute entries of the generalized precision matrix $\hat{\Omega}_{jk} \ \forall j$ %
    using $N$ samples
    \State Set $\hat{\Omega}_{kk} = 1$ %
    \EndFor\\
    \Return $\hat{\Omega}^{\textrm{L-SING}} := \frac{1}{2}(\hat{\Omega} + \hat{\Omega}^T)$
\end{algorithmic}
\end{algorithm}
In practice, we split the provided samples from the target distribution $\pi$ into training, validation, and estimation sets. The regularized objective in~\eqref{eq:tm} is optimized using the training set, while the unregularized negative log-likelihood is evaluated on the validation set to select the optimal regularization parameter $\lambda$. To ensure model generalization, we implement early stopping during the estimation of $S^k$. That is, training stops if the validation loss fails to improve for 10 consecutive epochs. Finally, we evaluate $\hat{\Omega}$ using the estimation set to avoid biases arising from reusing samples for both learning the map and computing the generalized precision.

\subsection{Edge Set Generation}

After computing and normalizing the generalized precision matrix $\hat{\Omega}$ (scaled to have maximum value 1), we generate a sparse edge set for the graphical model by thresholding:
\begin{enumerate}
    \item Choose a threshold value $\tau > 0$.
    \item For each pair of variables $(j,k)$ for all $j,k \in [1, d]$:
\begin{itemize}
    \item If $|\hat{\Omega}_{jk}| > \tau$, add an edge between variables $j$ and $k$.
    \item Otherwise, no edge is added.
\end{itemize}
\end{enumerate}
Given that the choice of threshold $\tau$ affects the sparsity of the graph, we demonstrate the sensitivity of the graph sparsity and false positive rates (FPR) to variations in $\tau$ in our numerical experiments. In practice, $\tau$ can be selected based on prior knowledge about the expected graph density.

\section{Numerical Results} \label{sec:exp}

We now aim to answer the following questions: (1) Can L-SING accurately quantify the conditional dependencies of $\mathbf{X}$ without relying on assumptions about the distribution of $\mathbf{X}$? (2) Is L-SING computationally tractable for high-dimensional problems? The first and second experiments address question (1), while the second and third experiments address question (2). Additionally, we compare the performance of $\hat{\Omega}^{\text{L-SING}}$ to existing methods on the same test dataset. Our code is publicly available at \url{https://github.com/SarahLiaw/L-SING}; details of the experiments not mentioned in the main text are described in Appendix~\ref{app:exp_details}.

\subsection{Gaussian Distribution}

To validate L-SING against existing parametric methods, we evaluated it on data sampled from a Gaussian distribution. We first generated a symmetric, positive definite matrix $\Sigma^{-1} = \Omega$, where $\Omega$ is the true precision matrix. Using this distribution, we drew $M = 5,000$ training samples from a $d$-dimensional multivariate normal distribution $\mathbf{X} \sim \mathcal{N}(\mathbf{0}, \Sigma)$, with $d = 10$. After training L-SING, we computed $\hat{\Omega}$ using an evaluation set of $N = 10,000$ samples drawn from the same distribution.

\begin{figure}[!ht]
    \centering
    \begin{subfigure}[b]{0.47\linewidth}
        \centering
        \includegraphics[width=\linewidth]{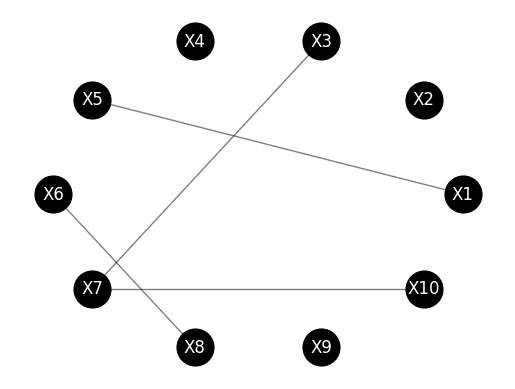}
         \caption{}
        \label{fig:sub2}
    \end{subfigure}
    \begin{subfigure}[b]{0.47\linewidth}
        \centering
        \includegraphics[width=0.75\linewidth]{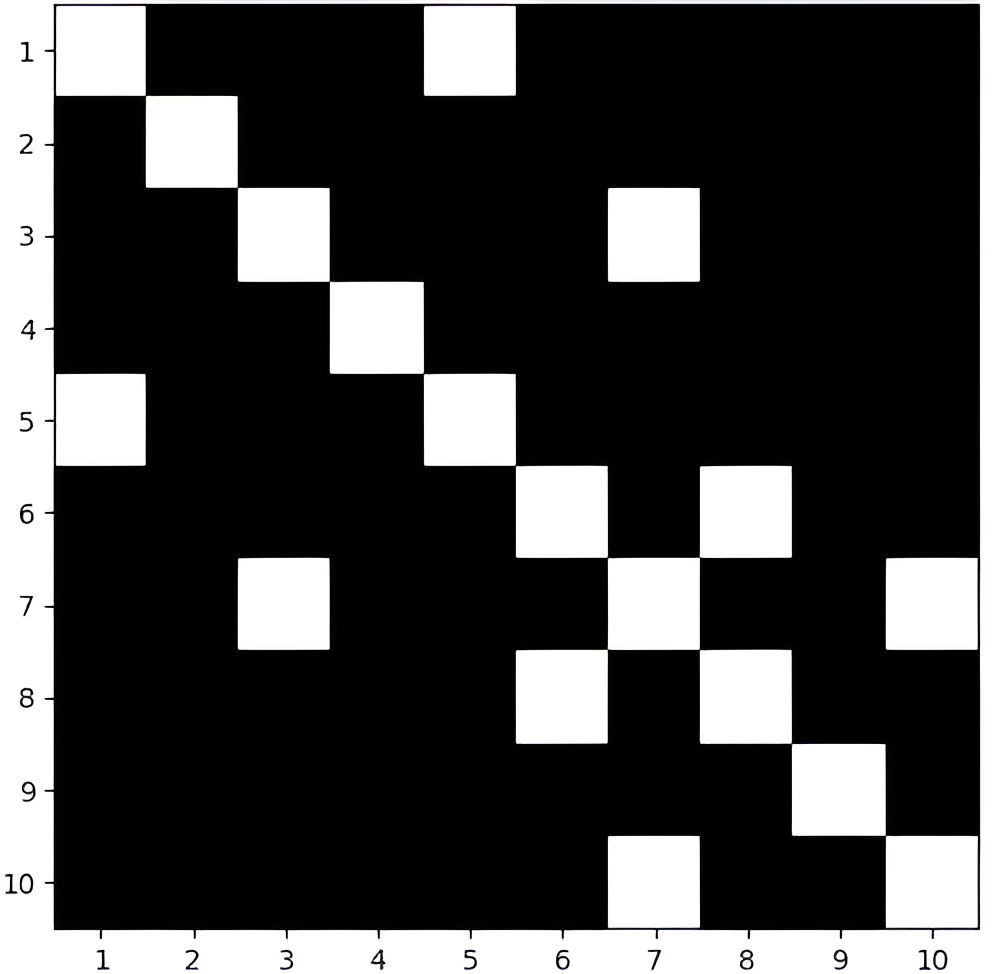}
        \caption{}
        \label{fig:sub1}
    \end{subfigure}
    \caption{(a) The undirected graphical model; (b) Adjacency matrix of true graph (white corresponds to an edge, black to no edge) for the 10-dimensional Gaussian distribution.}
    \label{fig:overall}
\end{figure}

\paragraph{Evaluation.}Figure~\ref{fig:overall} visualizes the true underlying graph and its adjacency matrix. Figure~\ref{fig:overall_gauss} compares the estimated precision matrices $\hat{\Omega}$ computed using L-SING and GLASSO.

\begin{figure}[!ht]
    \centering
    \begin{subfigure}[b]{0.4\linewidth}
        \centering
        \includegraphics[width=\linewidth]{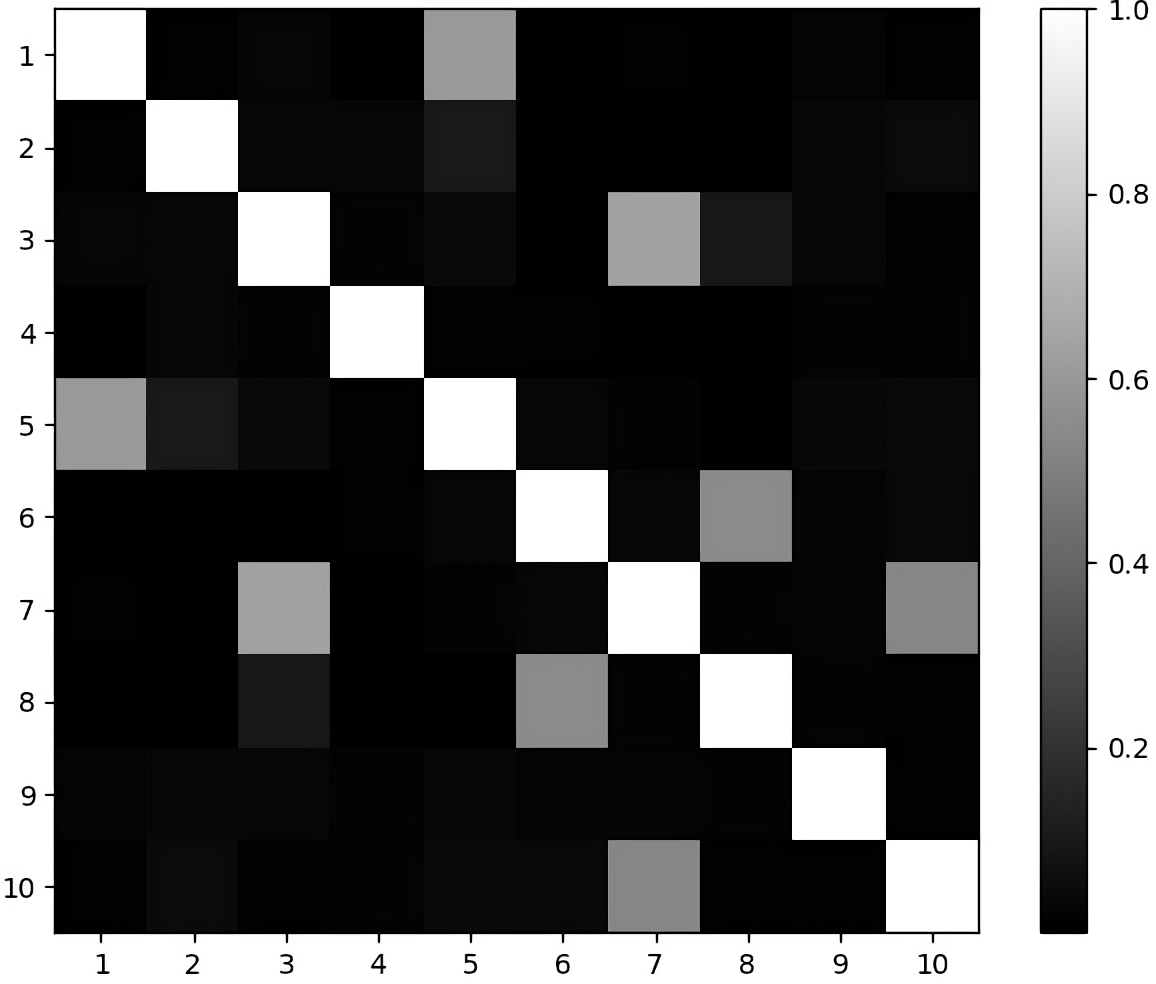}
         \caption{$\hat{\Omega}^{\text{L-SING}}$}
        \label{fig:compute_local_gauss}
    \end{subfigure}
    \hspace{0.5cm}
    \begin{subfigure}[b]{0.4\linewidth}
        \centering
        \includegraphics[width=\linewidth]{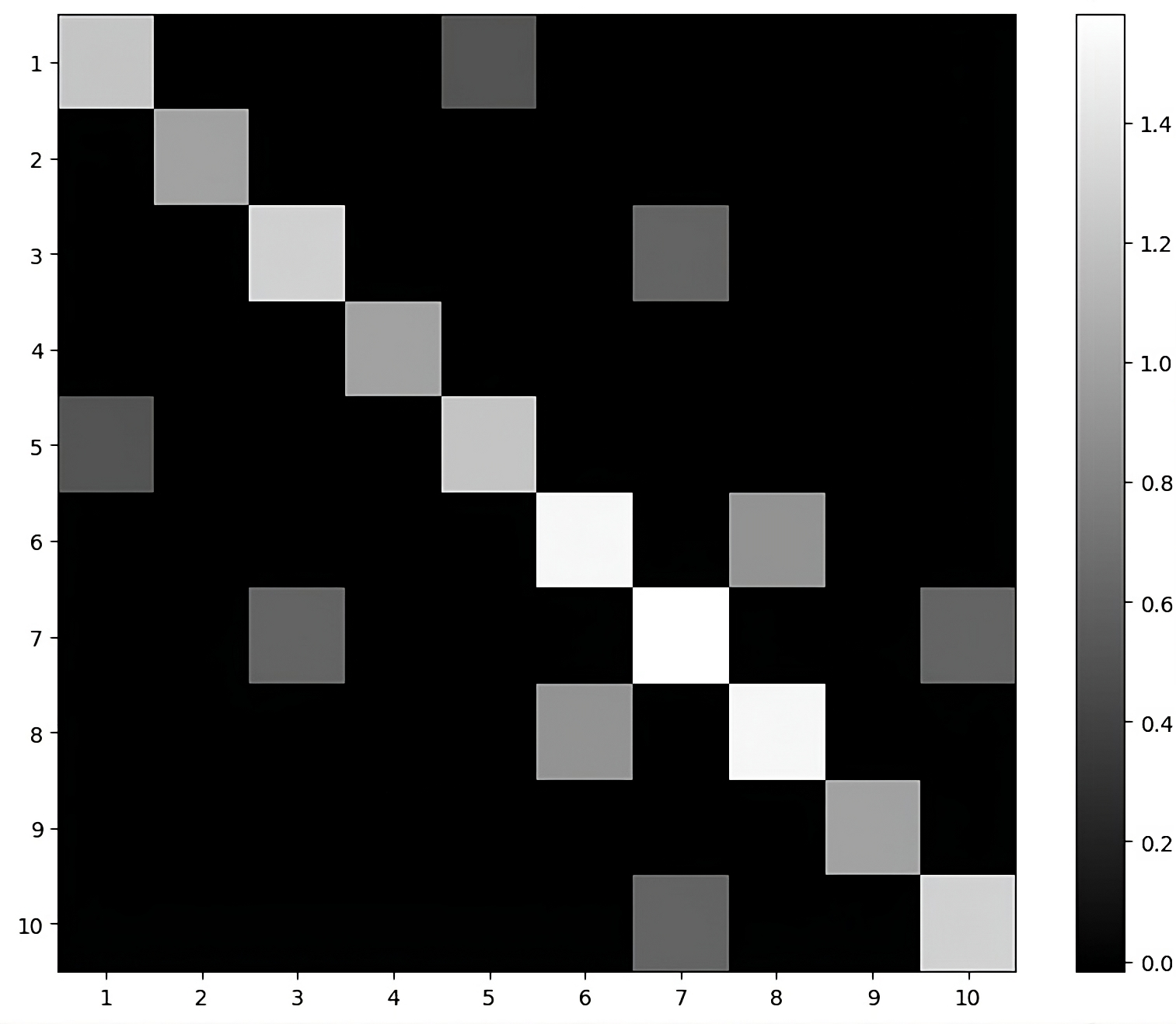}
        \caption{$\hat{\Omega}^{\text{GLASSO}}$}
        \label{fig:compute_glasso_gauss}
    \end{subfigure}
    \caption{Generalized precision matrix for L-SING and the estimated precision amtrix for GLASSO computed using $N = 10,000$ Gaussian evaluation samples.}
    \label{fig:overall_gauss}
\end{figure}

The heatmap colors in Figures~\ref{fig:compute_local_gauss} and~\ref{fig:compute_glasso_gauss} represent the magnitude of the entries in $\hat{\Omega}$. Figure~\ref{fig:compute_local_gauss} shows that L-SING accurately recovers the overall sparsity of $\Omega$. For $\tau = 0.2$, the estimated adjacency matrix matches the true adjacency matrix in Figure~\ref{fig:sub1}, with all edges correctly identified. The non-zero off-diagonal entries are within $\pm 0.1$ of the corresponding values in the true normalized precision matrix (Figure~\ref{fig:sub1}). 

Figure~\ref{fig:compute_glasso_gauss} presents results for the Gaussian simulation using GLASSO. Comparing the off-diagonal entries of $\hat{\Omega}^{\text{GLASSO}}$ to $\Omega$, edges \((1, 5)\) and \((3, 7)\) are within a $\pm 0.03$ range of their true values, while all other edges fall within $\pm 0.1$. Note that GLASSO operates directly on the test samples without requiring training. Figure~\ref{fig:sensitivity_gauss} plots the FPR against the number of training samples for $\tau = 0.20, 0.10, 0.05$. L-SING's FPR decreases with increasing $M$, which demonstrates its consistency. Higher $\tau$ (e.g., $0.20$) results in lower FPR across all sample sizes, indicating more conservative edge detection. The choice of $\tau$ results in a trade-off between sensitivity and specificity in edge detection.

\begin{figure}[h!]
    \centering
    \includegraphics[width=0.8\linewidth]{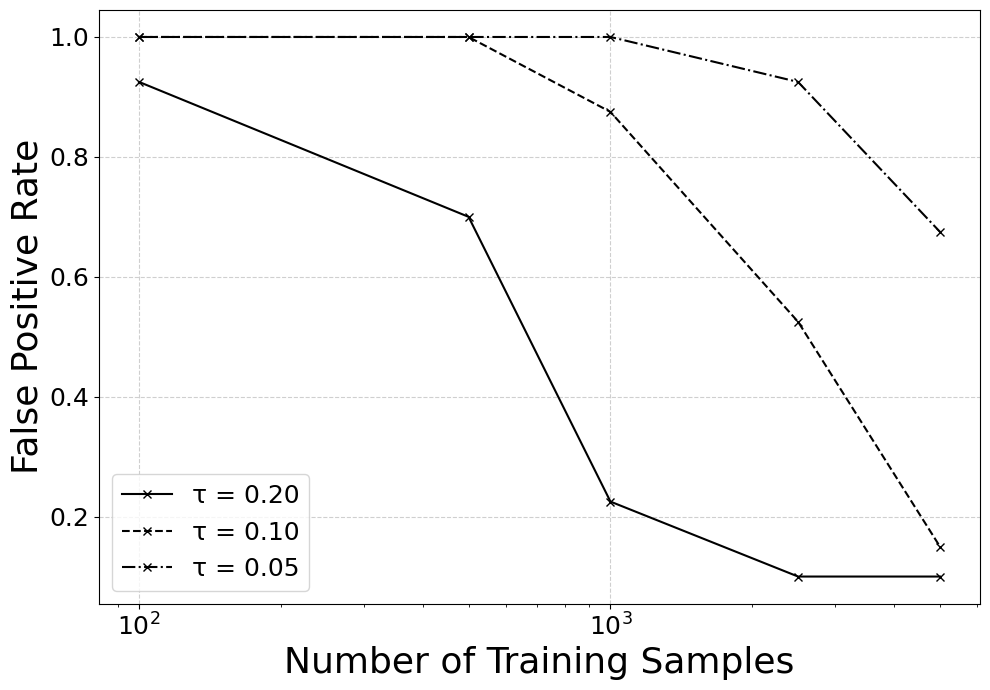}
    \caption{Sensitivity of false positives with L-SING for different thresholds and sample sizes on the Gaussian data.}
    \label{fig:sensitivity_gauss}
\end{figure}
\subsection{Butterfly Distribution}
Next, we evaluate L-SING on a non-Gaussian butterfly distribution, which exhibits nonlinear dependencies. Consider \( r \) pairs of random variables \( (X, Y) \), where:
\[
X \sim \mathcal{N}(0, 1)  \quad \quad 
Y = WX, \ \  \text{with } W \sim \mathcal{N}(0, 1). 
\]
Figure~\ref{fig:bfly_true} displays the probabilistic graph and the corresponding support of the generalized precision matrix for $r = 5$ pairs ($d=10$).
The variables are ordered $X_1, Y_1, \hdots, X_r, Y_r$, where $X_i$ corresponds to odd-numbered columns/rows in the heatmap and $Y_i$ to even-numbered ones for all $i \in [1,r]$. This ordering is consistent across all plots of the identified graph. While each one-dimensional marginal of the butterfly distribution is symmetric and unimodal, the two-dimensional marginals exhibit strong non-Gaussianity.

\begin{figure}[ht]
    \centering
    \begin{subfigure}[c]{0.4\linewidth}
        \centering
        \includegraphics[width=\linewidth]{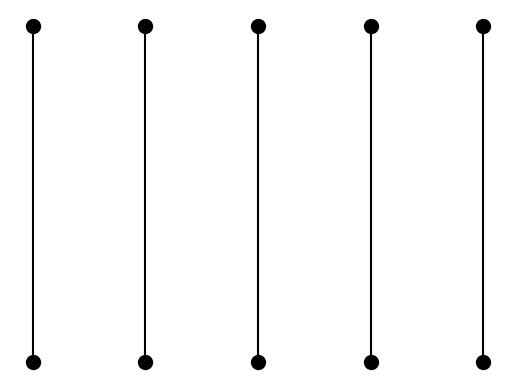}
        \caption{}
        \label{fig:bfly_imap}
    \end{subfigure}
    \hspace{0.5cm}
    \begin{subfigure}[c]{0.45\linewidth}
        \centering
        \includegraphics[width=0.75\linewidth]{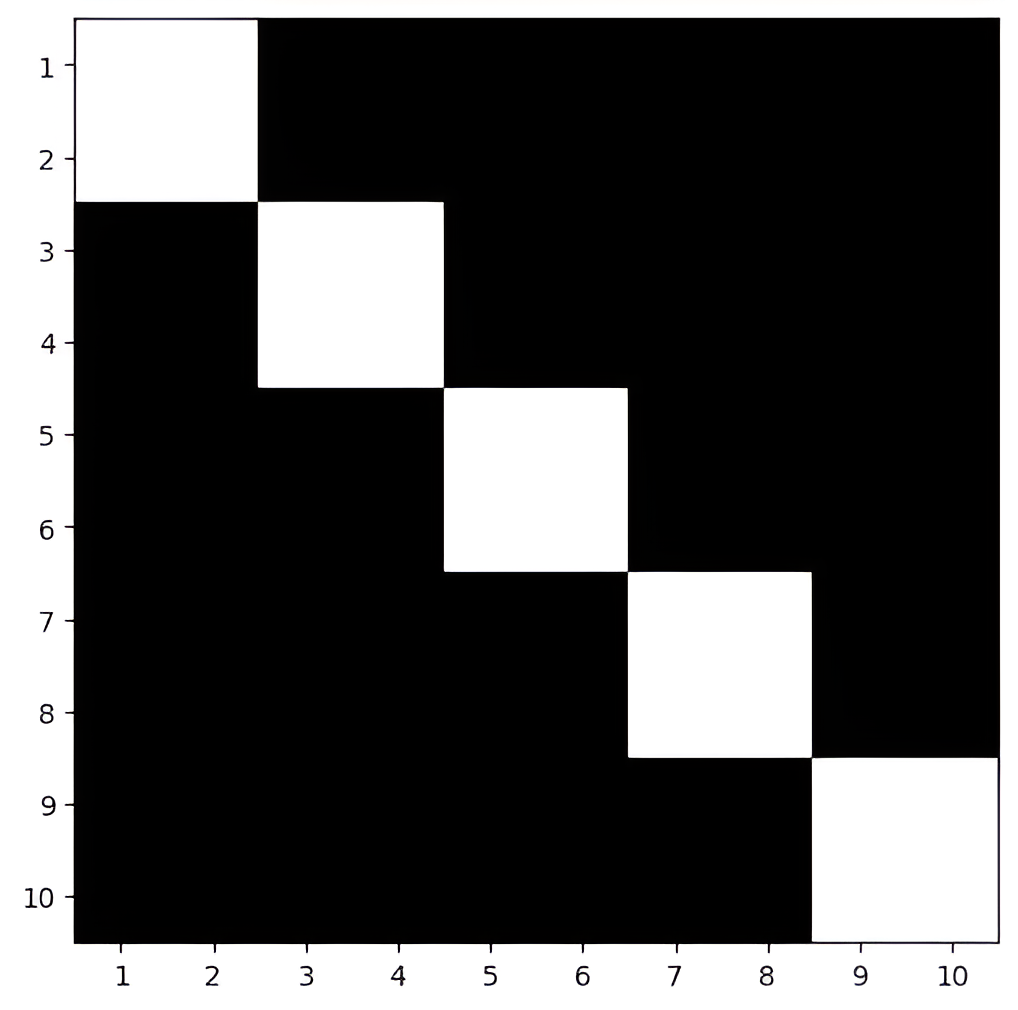}
        \caption{}
        \label{fig:bfly_adj_mat}
    \end{subfigure}
    \caption{(a) The undirected graphical model; (b) Adjacency matrix of true graph for the butterfly distribution.}
    \label{fig:bfly_true}
\end{figure}

\paragraph{Evaluation.}Figure~\ref{fig:sub3} shows the estimated generalized precision matrix $\hat{\Omega}$ for $r = 5$ pairs ($d = 10$), as computed using UMNN map components with $[64, 64, 64]$ hidden layers and $M = 5,000$ training samples. L-SING successfully recovers the true sparse structure of the graphical model for $\tau = 0.2$ and $0.1$, as shown in Figure~\ref{fig:butterfly_5p_sensitivity}.

To show the scalability of L-SING, Figure~\ref{fig:sub4} presents $\hat{\Omega}$ for $r = 20$ pairs ($d = 40$), using $M = 5,000$ training samples and the same UMNN architecture. With $\tau = 0.1$, L-SING achieves an $F_1$ score of approximately $0.941$, indicating high precision and recall in identifying the correct edges, even in higher dimensions. The FPR, approximately $6.58\times 10^{-3}$, highlights that L-SING introduces minimal spurious edges. These results show that L-SING accurately recovers the structure of the butterfly distribution, even as the dimensionality increases.\par 

\begin{figure}[!ht]
    \centering
    \begin{subfigure}[b]{0.4\linewidth}
        \centering
        \includegraphics[width=0.91\linewidth]{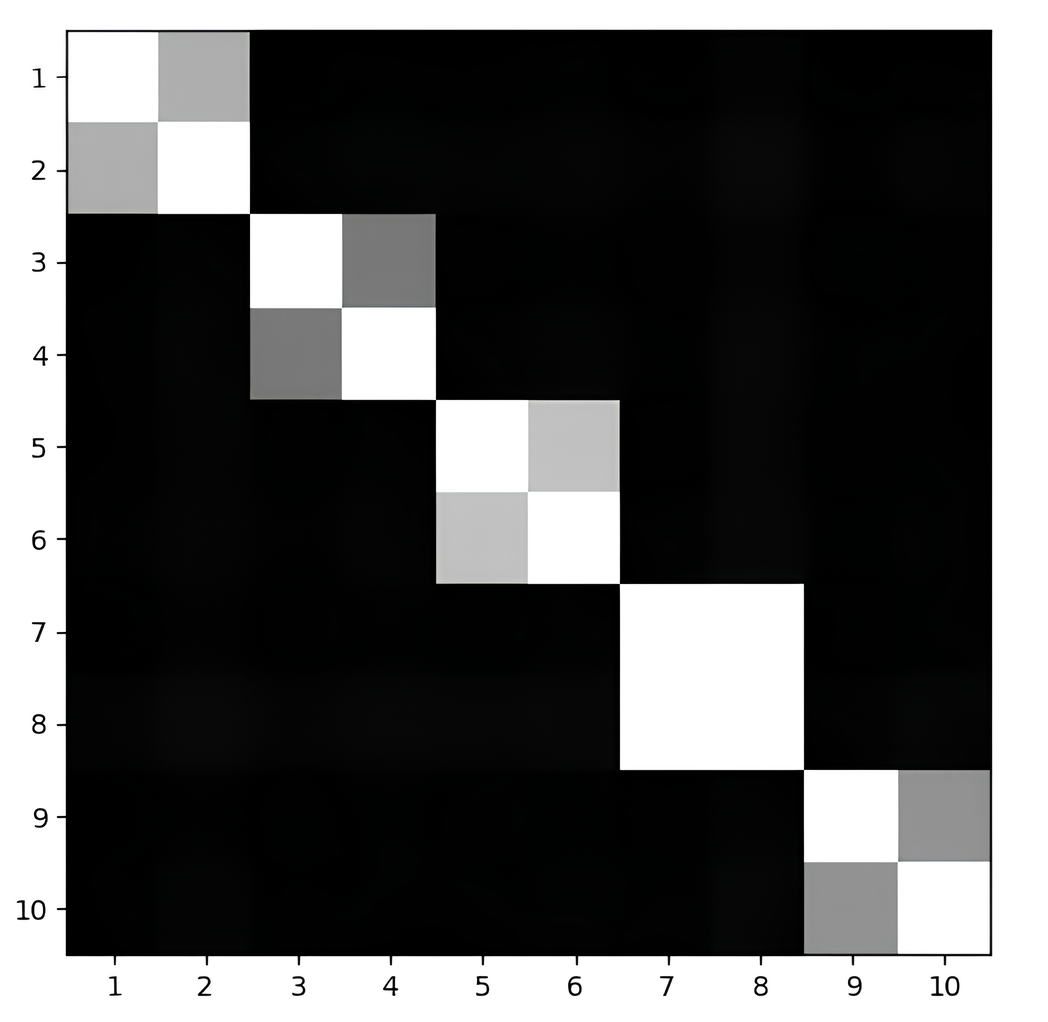}
        \caption{$\hat{\Omega}^{\text{L-SING}}$ for $d=10$}
        \label{fig:sub3}
    \end{subfigure}
    \hspace{0.5cm}
    \begin{subfigure}[b]{0.4\linewidth}
        \centering
        \includegraphics[width=1.03\linewidth]{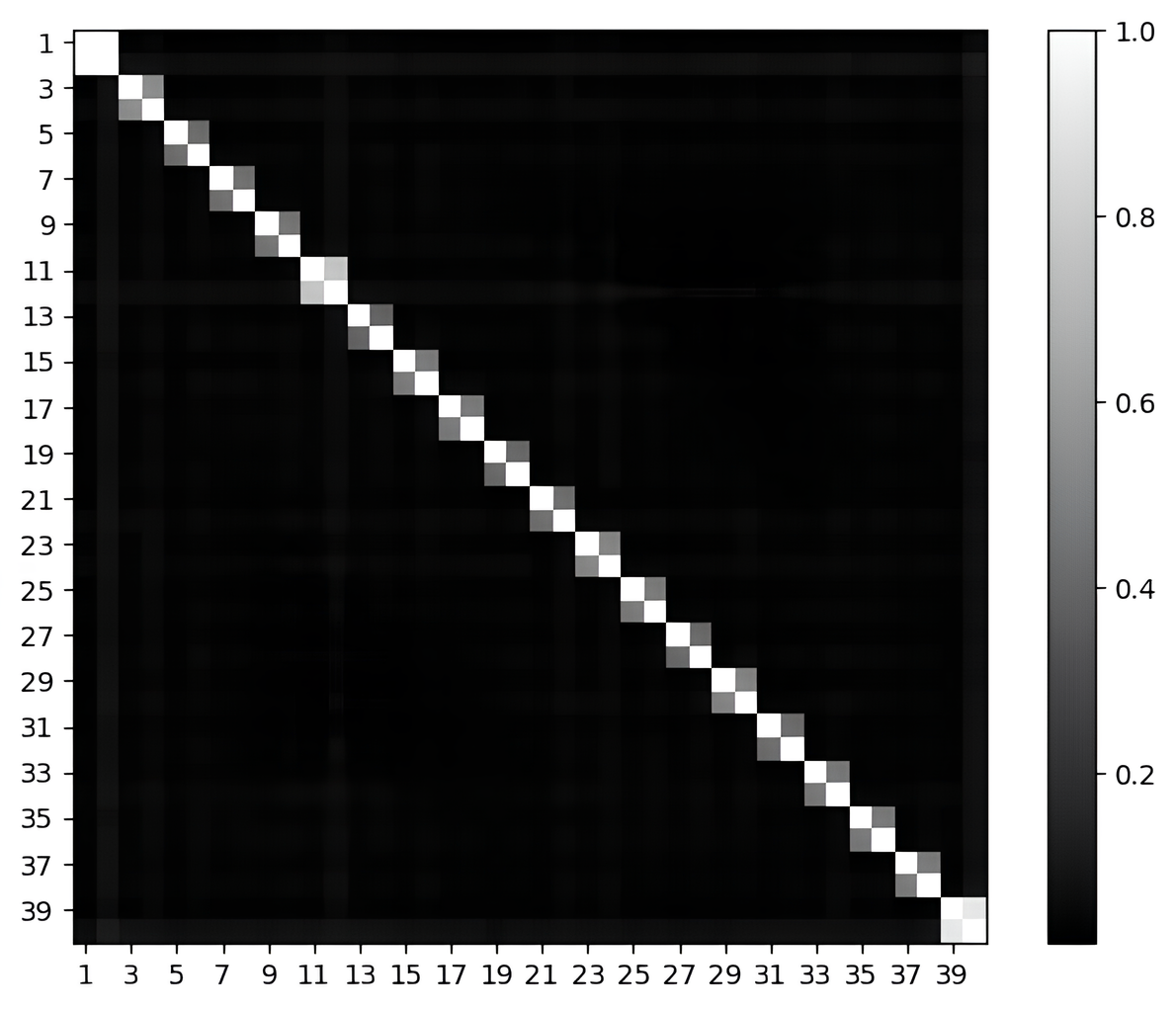}
        \caption{$\hat{\Omega}^{\text{L-SING}}$ for $d=40$}
        \label{fig:sub4}
    \end{subfigure}
    \caption{Estimated generalized precision matrix for the Butterfly distribution with $d=10$ and $d=40$ variables using L-SING with $N = 10,000$ evaluation samples.}
    \label{fig:overall2}
\end{figure}

\begin{figure}[!ht]
    \centering
    \includegraphics[width=0.8\linewidth]{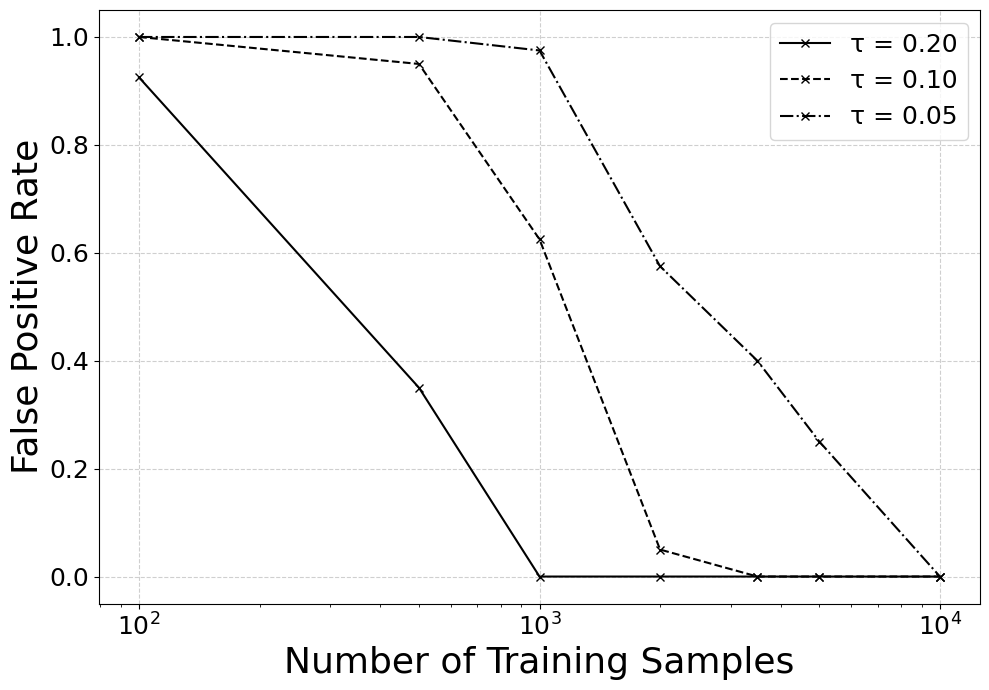}
    \caption{Sensitivity of false positives with L-SING on the 10-dimensional butterfly distribution.}
    \label{fig:butterfly_5p_sensitivity}
\end{figure}

Figure~\ref{fig:bfly_others} shows that both GLASSO and the nonparanormal (npn) method yield incorrect graphs for this dataset. These methods only identify the diagonal entries in Figure~\ref{fig:bfly_adj_mat}, corresponding to self-dependencies of each variable, and fail to recover the  dependencies between $X_i$ and $Y_i$ for all $i \in [1, r]$. GLASSO's failure arises from its reliance on a Gaussian assumption, which does not hold for the butterfly distribution. Similarly, the nonparanormal method, which applies a truncated ECDF transformation to each variable's marginal distribution before running GLASSO, fails because the butterfly distribution lies outside the class of distributions assumed by~\citet{JMLR:v10:liu09a}. Consequently, it cannot recover the true edges between variable pairs.

\begin{figure}[!ht]
    \centering
    \begin{subfigure}[b]{0.46\linewidth}
        \centering
        \includegraphics[width=0.9\linewidth]{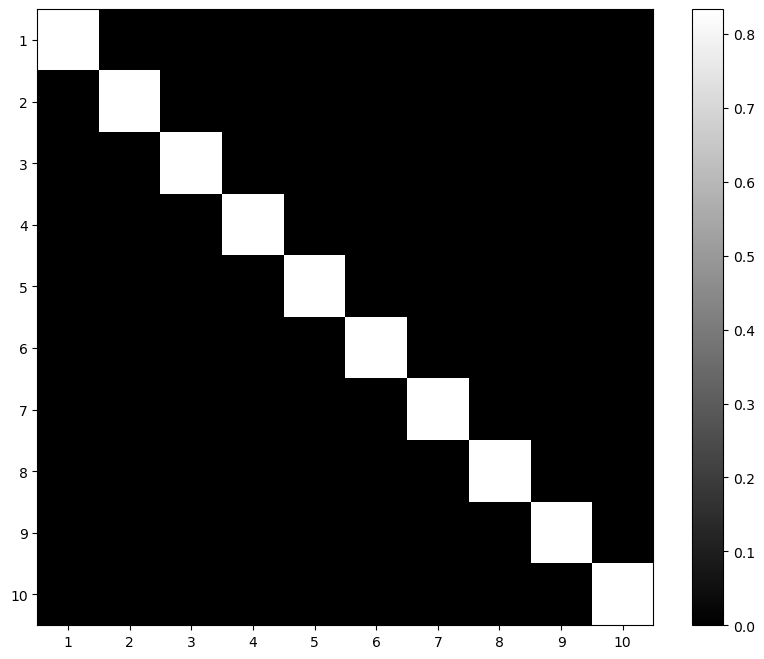}
        \caption{$\hat{\Omega}^{\text{npn}}$}
        \label{fig:bfly_npn}
    \end{subfigure}
    \begin{subfigure}[b]{0.46\linewidth}
        \centering
        \includegraphics[width=0.9\linewidth]{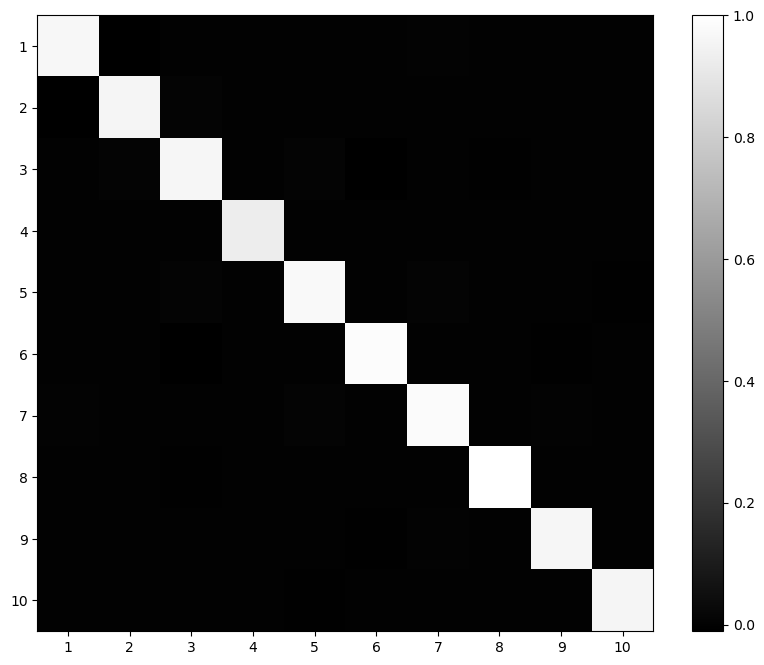}
        \caption{$\hat{\Omega}^{\text{GLASSO}}$}
        \label{fig:bfly_glasso}
    \end{subfigure}
    \caption{Conditional independencies with (a) the nonparanormal and (b) GLASSO for the butterfly distribution.}
    \label{fig:bfly_others}
\end{figure}

\subsection{Ovarian Cancer Dataset}\label{sec:oc_dataset}
Finally, we address question (2) by demonstrating the scalability of L-SING on the high-dimensional curated Ovarian Data~\citep{Ganzfried2013}, comprising gene expression profiles from 578 ovarian cancer patients sourced from The Cancer Genome Atlas (TCGA). Following the data processing procedure in~\citet{Shutta2022}, we selected biologically relevant genes. Specifically, we identified two gene sets from the Molecular Signatures Database: genes down- and up-regulated in mucinous ovarian tumors compared to normal ovarian epithelial cells. After intersecting these genes with those available in TCGA, the final dataset included 156 genes (variables) and 578 samples, split into 346 training, 117 evaluation, and 115 validation samples. For L-SING, estimating $S^k$ and computing the corresponding the matrix $\hat{\Omega}_k$ took 42.3 seconds, while GLASSO took 13.1 seconds.

\paragraph{Evaluation. }Figure~\ref{fig:lsing_oc} presents $\hat{\Omega}$, as computed using UMNN map components with  $[64, 128, 128]$ hidden layers. Figure~\ref{fig:lsing_tau_oc} shows the effect of $\tau$ on graph sparsity, showing that larger $\tau$ leads to sparser graphs, with the most significant changes in edge count and sparsity occurring for $\tau \in [0, 0.4]$. This range is critical for estimating the graph. Rather than relying on manual graph inspection, we follow the threshold selection procedure in~\citet{Shutta2022} and set $\tau = 0.2$ based on sparsity and edge count. Notably, $58\%$ of the entries in $\hat{\Omega}^{\text{L-SING}}$ are less than $0.2$, resulting in a sparse graph as weak connections are effectively pruned.

\begin{figure}[!ht]
    \centering
    \begin{subfigure}[c]{0.45\linewidth}
        \centering
        \includegraphics[width=\linewidth]{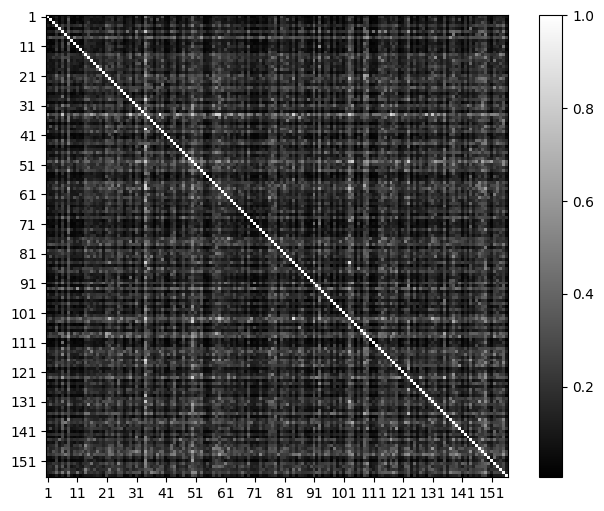}
        \caption{$\hat{\Omega}^{\text{L-SING}}$}
        \label{fig:lsing_oc}
    \end{subfigure}
    \begin{subfigure}[c]{0.54\linewidth}
        \centering
        \includegraphics[width=\linewidth, height=4.6cm]{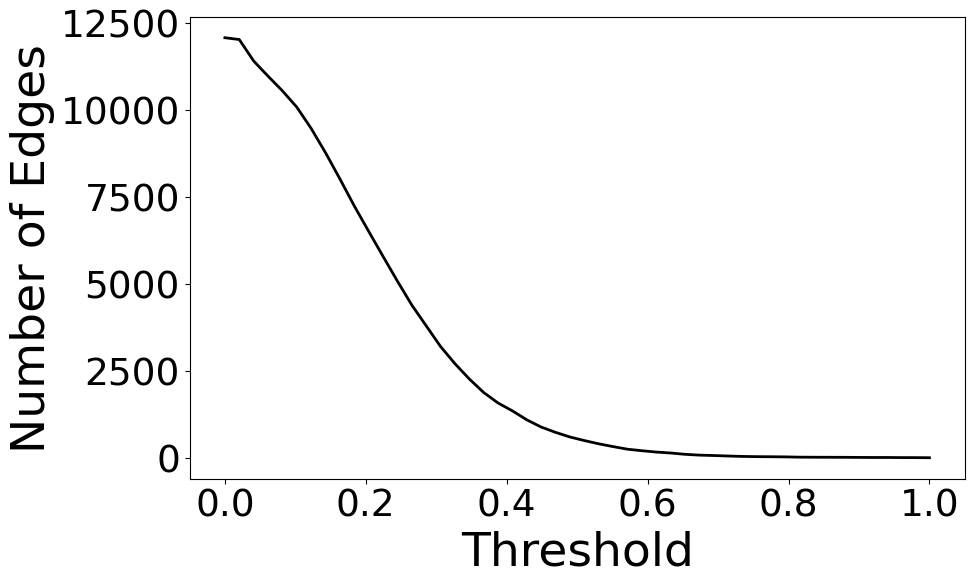}
        \caption{}
        \label{fig:lsing_tau_oc}
    \end{subfigure}
    \caption{(a) Generalized precision for the Ovarian Cancer dataset; (b) Sensitivity of recovered edge count to changes in the threshold $\tau$.}
    \label{fig:OC_ls}
\end{figure}

To avoid biases in model interpretation due to qualitative or visual inspection, we adopt a quantitative approach using centrality measures to compare against existing methods, given the absence of a ground truth graph for this problem. Specifically, we compute an average centrality rank based on betweenness, degree, hubscore, and closeness centralities. Using the support of the generalized precision matrix computed by L-SING (Figure~\ref{fig:lsing_oc}), we observe that the gene CTSE exhibits the highest mean centrality rank. This result aligns with prior findings by~\citet{Marquez2005}, which identify CTSE as an up-regulated and specific marker for mucinous ovarian cancers.

We compare these results to those of~\citet{Shutta2022}, who use GLASSO (Figure~\ref{fig:glasso_scaled_oc}). Their analysis highlights EPCAM as the second most important gene in ovarian cancer, which is supported by~\citet{Spizzo2006}, whereas CTSE is not identified as significant. In contrast, L-SING ranks EPCAM 14th out of 156 genes based on mean centrality. A direct comparison between L-SING and GLASSO is challenging due to the lack of a ground truth, as L-SING captures additional non-Gaussian dependencies that GLASSO, by design, cannot detect.

Next, we compare $\hat{\Omega}^{\text{GLASSO}}$ and $\hat{\Omega}^{\text{npn}}$, where $\hat{\Omega}^{\text{npn}}$ is obtained by applying a truncated ECDF transformation to the data before applying GLASSO. The difference, measured by the Frobenius norm, is $\|\hat{\Omega}^{\text{GLASSO}} - \hat{\Omega}^{\text{npn}}\|_F = 6.43$, indicating that relying solely on Gaussian-based methods provides an incomplete representation of the network structure. Our findings demonstrate that parametric methods like GLASSO may fail to detect non-Gaussian relationships, which underscores the advantage of using L-SING for recovering nonlinear dependencies.

\begin{figure}[!ht]
    \centering
    \begin{subfigure}{0.45\linewidth}
        \centering
        \includegraphics[width=0.84\linewidth]{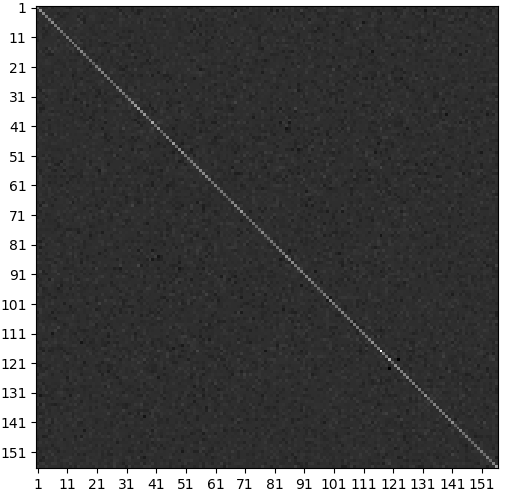}
    \caption{$\hat{\Omega}^{\text{GLASSO}}$}
        \label{fig:glasso_scaled_oc}
    \end{subfigure}%
    \begin{subfigure}{0.45\linewidth}
        \centering
        \includegraphics[width=\linewidth]{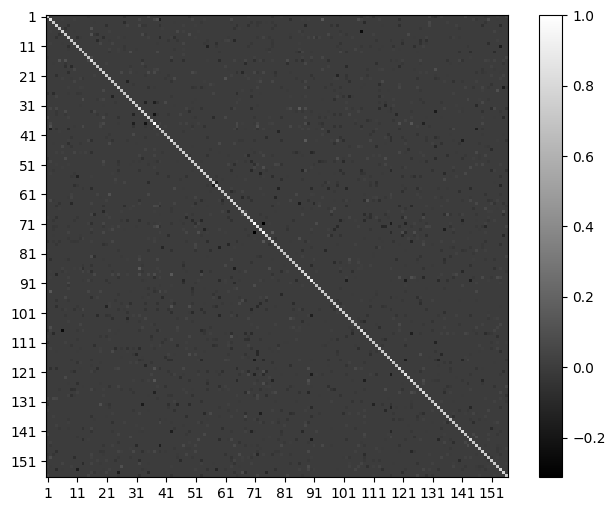}
        \caption{$\hat{\Omega}^{\text{npn}}$}
        \label{fig:npn_scaled_oc}
    \end{subfigure}
    \caption{Estimated pairwise conditional independencies on evaluation samples from the  Ovarian Cancer problem.}
    \label{fig:combined_figure}
\end{figure}

\section{Conclusion}
We have proposed L-SING, a method for learning the structure of high-dimensional graphical models underlying non-Gaussian distributions using transportation of measures. Unlike previous methods that estimate the joint distribution all at once, L-SING constructs a generalized precision matrix by learning each variable’s neighborhood independently. This local approach allows for parallelization, making L-SING computationally tractable (e.g., more memory efficient than global methods) in high-dimensional settings. We have also shown the broad applicability of L-SING by establishing theoretical connections to existing methods and through empirical comparisons. Future work includes extending L-SING to handle mixed continuous and discrete variables, as well as developing consistent thresholding and scoring strategies to reduce sensitivity to tuning parameters—similarly to the approach in~\citet{zhao2024highdimensionalfunctionalgraphicalmodel} for the Gaussian setting.

\section*{Acknowledgements}
SL acknowledges support from the Citadel Global Fixed Income SURF Endowment. YM acknowledges support from DOE ASCR award DE-SC0023187 and from the Office of Naval Research under award N00014-20-1-259. RB is grateful for support from the von K\'{a}rm\'{a}n instructorship at Caltech, the Air Force Office of Scientific Research MURI on “Machine Learning and Physics-Based Modeling and Simulation” (award FA9550-20-1-0358) and a Department of Defense (DoD) Vannevar Bush Faculty Fellowship (award N00014-22-1-2790) held by Andrew M.\ Stuart. 
\bibliography{references}

\appendix

\section{Theoretical Details} \label{app:theoretical_dets}

\begin{theorem} The optimization problem for learning the map component $S^k$:
\begin{align} \label{eq:OPT_problem_appendix}
\min_{S^k} \mathbb{E}_{\pi(\mathbf{x})} \left[\frac{1}{2} (S^k)^2(\mathbf{x}) - \log \partial_k S^k (\mathbf{x})\right] \\
\text{s.t. } S^k \in \mathcal{S}_k; \ \ \partial_k S^k >0, \ \ \ \ \ \  (\pi - \text{a.e.}), \nonumber
\end{align}
corresponds to minimizing the expected Kullback-Leibler (KL) divergence objective in~\eqref{eq:KLobjective} from the standard Gaussian reference distribution $\eta_k = \mathcal{N}(0,1)$ to  pushforward of the target conditional distribution $S^k(\cdot,x_{-k})_\sharp \pi(\cdot|x_{-k})$.
\end{theorem}

\begin{proof} Recall the objective for the map in~\eqref{eq:KLobjective} is given by the functional $$S^k \mapsto \mathbb{E}_{\pi(x_{-k})}[D_{\text{KL}}(S^k(x_{-k},\cdot)_{\sharp} \pi(\cdot|x_{-k}) || \eta_k)].$$
From the invariance of the KL divergence to invertible transformations, the objective is equivalent to
$$S^k \mapsto \mathbb{E}_{\pi(x_{-k})}[D_{\text{KL}}(\pi(\cdot|x_{-k}) || S^k(x_{-k},\cdot)^\sharp\eta_k)].$$
Using the definition for the KL divergence, the objective can be expressed as 
\begin{align*}
\mathbb{E}_{\pi(x_{-k})}[D_{\text{KL}}(\pi(\cdot|x_{-k}) || S^k(x_{-k},\cdot)^\sharp \eta_k)] &= \mathbb{E}_{\pi(x)}\left[\log \pi(x_k|x_{-k})] - \mathbb{E}_{\pi(x)}[\log S^k(x_{-k},\cdot)^\sharp \eta(x_k) \right].
\end{align*}
Given that the first term is independent of $S$, minimizing the objective is equivalent to maximizing the second term, which is the  log-likelihood of the pullback distribution in expectation over the data distribution $\pi$. 

Then, choosing $\eta_k$ to be the standard isotropic Gaussian distribution $\mathcal{N}(0, 1)$, the second term in the objective becomes:
\begin{align*}
   -\mathbb{E}_{\pi(x)}[\log S^k(x_{-k},\cdot)^\sharp \eta(x_k)] &= -\mathbb{E}_{\pi(x)}\left[\log \left(\eta_k \circ S^k(x) \cdot |\partial_k S^k(x)|\right)\right] \\
   &= -\mathbb{E}_{\pi(x)} \left[\log\left(\frac{1}{\sqrt{2\pi}}\right) - \frac{1}{2}(S^k(x))^2 + \log|\partial_{x_k} S^k(x)|\right].
\end{align*}
 We observe that $\log \frac{1}{\sqrt{2\pi}}$ is constant with respect to $S^k$. Thus, we can minimize the KL divergence in~\eqref{eq:KLobjective} by minimizing the second and third terms above. 
\end{proof}

We note that when samples of $\pi$ are given, we can use using Monte Carlo to approximate the expectation in~\eqref{eq:OPT_problem_appendix}. 

The optimization problem with the empirical objective function for learning the transport map component $S^k$ becomes:
\[
\min_{S^k} \frac{1}{M} \sum_{i=1}^M\frac{1}{2} \left(S^k\right)^2(\mathbf{x}^i) - \log \partial_k S^k(\mathbf{x}^i).
\]
This corresponds to the objective in~\eqref{eq:tm} without the regularization term, i.e., $\lambda = 0$.

\section{Experimental Details} \label{app:exp_details}
In this section we detail additional supporting information for the computational experiments presented in Section~\ref{sec:exp}.

\subsection{Gaussian Experiments}
To validate L-SING, we conducted experiments using  synthetic Gaussian datasets that are drawn according to  distributions with known conditional independence structure due to the availability of the ground truth graph for comparison. The datasets were generated as follows:
\begin{enumerate} \itemsep-1pt
    \item Construct a precision matrix $\Theta \in \mathbb{R}^{d \times d}$ using the function \texttt{make\_sparse\_spd\_matrix} in \texttt{scikit-learn}, where the probability of an entry being zero is set to $0.95$ and the nonzero entries are drawn uniformly from the range $[0.3, 0.8]$. %
    We %
    take the covariance matrix to be the inverse of $\Theta$, i.e., $\Sigma = \Theta^{-1} $. %
    \item $M$ i.i.d.\thinspace samples \(\{\mathbf{x}^{i}\}_{i=1}^M\) were generated from \(\mathcal{N}(\mu, \Sigma)\), where $\mu$ was set to $\mathbf{0}$. 
    \item The data were centered and standardized, normalizing each feature to have zero mean and unit variance.
\end{enumerate}
For these experiments, we set the dimension \(d = 10\), allowing a tractable comparison between the L-SING and other algorithms for a fair evaluation of computational efficiency.

\subsubsection{Implementation Details}
All experiments were implemented in Python 3.10 and performed on a Linux system with a CPU and 16GB of RAM. %
To ensure that the experiments are reproducible, we fixed the random seed %

We parameterized the UMNN using a neural network architecture with three hidden layers, each containing 64 hidden units, following the implementation in~\citet{NEURIPS2019_2a084e55}. The integral for the transport map was evaluated using a Clenshaw-Curtis quadrature rule with 21 nodes. %
Prior experiments by~\citet{NEURIPS2019_2a084e55} indicated that the number of quadrature nodes has a minimal effect on the accuracy.

While we tested larger UMNN architectures, e.g., with 128 hidden units in each of the 3 layers, these yielded only marginal improvements in negative log-likelihood given by \(-0.1 \pm 0.05\) and negligible changes in the estimated precision matrix. In contrast, reducing the model to fewer than three hidden layers increased the negative log-likelihood by approximately \(+0.15 \pm 0.05\).

In our experiments, we considered the following sets of hyperparameters:
\begin{itemize} \itemsep0pt
    \item Training sizes: \( M \in \{100, 500, 1000, 2500, 5000\} \)
    \item Regularization values: \( \lambda \in \{1, 0.1, 0.01, 0.001, 0\} \)
    \item Threshold values: \( \tau \in \{0.2, 0.1, 0.05\} \)
\end{itemize}

The final results, including those presented in Figure~\ref{fig:overall_gauss}, were produced using a training set of size \(M = 5000\), which achieved a balance between computational efficiency and accuracy of the recovered generalized precision matrix.

\subsubsection{Procedure}
Each experiment followed the workflow outlined in L-SING (Algorithm~\ref{alg:lsing}):
\begin{enumerate} \itemsep0pt
    \item Generate a training set of size \(M\).
    \item Generate validation and evaluation sets, each containing \(N = 10,000\) samples.
    \item Use the training set to learn the transport map for each conditional distribution.
    \item Tune the hyperparameter \(\lambda\) using the validation set.
    \item Compute the generalized precision matrix by evaluating the model for each conditional distribution on the evaluation set.
\end{enumerate}

\subsubsection{Evaluation Metrics}
Two metrics were used to assess the quality of the learned transport map and recovery of the graph structure. The negative log-likelihood (NLL) was used to evaluate the quality of the learned transport map, while the false positive rate (FPR) was used to measure the accuracy of edge recovery in the graph structure with respect to the threshold \(\tau\).

We conducted one run per hyperparameter configuration. The computational time varied according to the size of the training set and the architecture of the model. Generally, larger models and datasets required more time, but provided only marginal improvements, making smaller architectures and moderate training sizes preferable.

\subsection{Butterfly Experiments}  

The butterfly distribution, as described in \citet{NIPS2017_ea8fcd92}, was used to evaluate the tractability of L-SING in higher dimensions and for direct comparison with SING. This distribution displays strongly non-Gaussian features, making it a standard benchmark in the literature for graph recovery problems beyond simpler Gaussian settings.  

The considered dataset $\{\mathbf{x}^{i}\}_{i=1}^M$ for this problem was generated following the procedure detailed in the numerical results section. Two problem sizes were considered:  
\begin{itemize} \itemsep0pt
    \item $d = 10$ (with $r = 5$ pairs) for direct comparison with SING.  
    \item $d = 40$ (with $r = 20$ pairs) to test the scalability of L-SING to higher dimensions.  
\end{itemize}

\subsubsection{Implementation Details}  
All experiments were implemented in Python 3.10 and executed on a Linux system with a CPU and 16GB of RAM. 

The UMNN was parameterized with a neural network architecture consisting of three hidden layers, each containing 64 units. This architecture was chosen based on preliminary experiments similar to those in the Gaussian case.  

In these experiments, we considered the following hyperparameter ranges:  
\begin{itemize} \itemsep0pt
    \item Training sizes: $M \in \{100, 500, 1000, 2500, 5000\}$  
    \item Regularization values: $\lambda \in \{1, 0.1, 0.01, 0.001, 0\}$  
    \item Threshold values: $\tau \in \{0.2, 0.1, 0.05\}$ 
\end{itemize}

\subsubsection{Evaluation Metrics}  
To evaluate performance, we measured negative log-likelihood, false positive rate and F$_1$ score (for $d = 40$) for balanced precision and recall metrics of graph structure recovery.

The F$_1$ score was included for the 40-dimensional problem because the number of potential edges grows combinatorially with the dimension $d$. In this scenario, precision (the fraction of recovered edges that are true positives) and recall (the fraction of true edges successfully recovered) can behave very differently. Thus, the F$_1$ score provides a single measure of recovery performance that accounts for these two metrics.

\subsection{Ovarian Cancer: Gene Expression Dataset Experiments}

To demonstrate the practical applicability and scalability of L-SING, we applied L-SING to a high-dimensional gene expression dataset ($d = 156$ variables) from $578$ participants with ovarian cancer. We used the publicly available \texttt{curatedOvarianData} dataset from the R package of~\citet{Ganzfried2013}.%
This dataset allows for reproducibility, validation, and comparison with prior studies, such as~\citet{Shutta2022}. The focus on gene regulatory networks aligns with the discussion in the introduction, where structure learning is important in revealing disease mechanisms and therapeutic targets.

This dataset is a challenging task because of its relatively small sample size and high dimensionality, in contrast to the synthetic experiments where sample sizes could be freely adjusted through additional generation of new samples. The dataset was split as follows:
\begin{itemize} \itemsep0pt
    \item Training set: 346 samples
    \item Validation set: 115 samples
    \item Evaluation set: 117 samples
\end{itemize}

We followed the pre-processing steps described by \citet{Shutta2022}, using the author's publicly available code. The steps include:
\begin{enumerate} \itemsep0pt
    \item Loading the ovarian cancer dataset from the \texttt{curatedOvarianData} R package.
    \item Standardizing the gene expression data by computing the mean and standard deviation for each gene (i.e., for each variable) across all 578 participants. Then, each gene's expression values were normalized by subtracting the gene's mean and dividing by its standard deviation. Thus, for each gene, the standardized values have a mean of 0 and a standard deviation of 1 across the entire dataset. %
    \item Selecting two sets of genes from the Molecular Signatures Database, associated with genes regulated down and up in mucinous ovarian tumors.
    \item Intersecting the selected genes with those available in the TCGA ovarian cancer dataset.
    \item Extracting a final dataset of 156 genes (variables) and 578 samples (participants).
\end{enumerate}

\subsubsection{Experimental Setup}

The preprocessed dataset was used to evaluate L-SING as well as the nonparanormal, and GLASSO methods for recovering the generalized precision matrix $\hat{\Omega}$. All experiments were implemented in Python~3.10 and conducted on a Linux system with a CPU and 16GB of RAM. For additional pre-processing details in R, we refer readers to~\citet{Shutta2022}.

The hyperparameter ranges we considered for this experiment were:
\begin{itemize} \itemsep0pt
    \item Regularization values: $\lambda \in \{1, 0.1, 0.01, 0.001, 0\}$
    \item Threshold values $\tau$: Uniform grid of 50 values between 0 and 1
\end{itemize}
In Section~\ref{sec:oc_dataset}, we demonstrate the relationship between $\tau$ and graph sparsity in terms of edge counts. Using a fine grid of hyper-parameters was especially important for this dataset due to the small sample size. We note that the optimal threshold accounts for the higher variance and statistical error in both map estimation and the computed entries of the generalized precision matrix.

We used a larger UMNN architecture for this dataset consisting of three hidden layers with $64$, $128$, and $128$ hidden units. Despite the increased computational complexity, the smaller size for the training set allowed for efficient training. On average, the time required to learn the map component $S^k$ and compute the corresponding entries in the the $k$-th row/column of the generalized precision matrix was $42.3 \pm 0.05$ seconds for each $k$. Computing all $S^k$ components in parallel and independently makes L-SING tractable for high-dimensional problems.

\subsubsection{Evaluation Metrics}

For this dataset, we used $\tau = 0.2$ to compute the adjacency matrix, based on successful graph structure recovery in other experiments. %
Given that the dataset lacks a ground-truth graph, it is not possible to compute metrics such as False Positive Rate (FPR) and $F_1$ scores. Instead, we evaluated other commonly reported measures in the literature:
\begin{enumerate} \itemsep0pt
    \item Network measures: Degree, hub score, closeness, and betweenness.
    \item Gene ranking: Metrics were normalized through ranking and averaging, identifying each gene’s relative importance as discussed in \citet{Shutta2022}.
\end{enumerate}

Specifically, we compared the top-ranked genes identified by L-SING with those from GLASSO and verified their relevance relative to results reported in the literature. 

Finally, to assess the impact of non-Gaussian dependencies, we also compared GLASSO with and without a nonparanormal transformation in Section~\ref{sec:oc_dataset}. This analysis shows the importance of algorithms like L-SING, which do not assume specific form for the underlying data distribution. %

\end{document}